\def\year{2019}\relax
\documentclass[letterpaper]{article} 
\usepackage{aaai19}  
\usepackage{times}  
\usepackage{helvet}  
\usepackage{courier}  
\usepackage{url}  
\usepackage{graphicx}  
\usepackage{amsmath}
\usepackage{amsfonts}
\usepackage{amsthm}
\usepackage{tikz}
\usepackage{pgfplots}
\usepackage{booktabs}
\usepackage{multirow}
\usepackage{pbox}
\usepackage{threeparttable}

\usepgfplotslibrary{fillbetween}

\newtheorem{proposition}{Proposition}

\begin{document}
%
\title{On the Calibration of Nested Dichotomies for Large Multiclass Tasks}
\author{Tim Leathart, Eibe Frank, Bernhard Pfahringer and Geoffrey Holmes\\
Department of Computer Science, University of Waikato, New Zealand
}

\maketitle
\begin{abstract}
	Nested dichotomies are used as a method of transforming a multiclass classification problem into a series of binary problems. A tree structure is induced that recursively splits the set of classes into subsets, and a binary classification model learns to discriminate between the two subsets of classes at each node. In this paper, we demonstrate that these nested dichotomies typically exhibit poor probability calibration, even when the base binary models are well calibrated. We also show that this problem is exacerbated when the binary models are poorly calibrated. We discuss the effectiveness of different calibration strategies and show that accuracy and log-loss can be significantly improved by calibrating both the internal base models and the full nested dichotomy structure, especially when the number of classes is high.
\end{abstract}

\section{Introduction}
As the amount of data collected online continues to grow, modern datasets utilised in machine learning are increasing in size. Not only do these datasets exhibit a large number of examples and features, but many also have a very high number of classes. It is not uncommon in some application areas to see datasets containing tens of thousands~\cite{deng2009imagenet}, or even millions of classes~\cite{dekel2010multiclass,agrawal2013multi}. 

An attractive option to handle datasets with such large label spaces is to induce a binary tree structure over the label space. At each split node $k$, the set of classes present, $\mathcal{C}_k$, is split into two disjoint subsets $\mathcal{C}_{k1}$ and $\mathcal{C}_{k2}$. Then, a binary classification model is trained to distinguish between these two subsets of classes. Many algorithms have been proposed that fit this general description, for example~\cite{beygelzimer2009error,bengio2010label,choromanska2015logarithmic,daume17logarithmic}. Usually, in these tree structures, a greedy inference approach is taken, i.e., test examples only take a single path from the root node to leaf nodes. This has the inherent drawback that a single mistake along the path to a leaf node results in an incorrect prediction. 

In this paper, we consider methods with probabilistic classifiers at the internal nodes, called \textit{nested dichotomies} in the literature~\cite{frank2004ensembles}. Utilising probabilistic binary classifiers at the internal nodes to make routing decisions for test examples has several advantages over simply taking a hard 0/1 classification. For example, multiclass class probability estimates can be computed in a natural way by taking the product of binary probability estimates on the path from the root to the leaf node~\cite{fox1997applied}. However, although hard classification decisions are avoided, small errors in the binary probability estimates can accumulate over this product, resulting in overall potentially poorer quality predictions. Datasets with more classes result in deeper trees, exacerbating this issue. 

In this paper, we investigate approaches to reduce the impact of the accumulation of errors by utilising \textit{probability calibration} techniques. Probability calibration is the task of transforming the probabilities output by a model to reflect their true empirical distribution; for the group of test examples that are predicted to belong to some class with probability $0.8$, we expect about $80\%$ of them to actually belong to that class if our model is well calibrated. Our main hypothesis is that the overall predictive performance of nested dichotomies can be improved by calibrating the individual binary models at internal nodes (referred to in this paper as \textit{internal calibration}). However, we also observe that significant performance gains can be achieved by calibrating the entire nested dichotomy (referred to as \textit{external calibration}), even if the internal models are well calibrated.

This paper is structured as follows. First, we briefly review nested dichotomies and probability calibration. We then discuss internal and external calibration, providing theoretical motivation and showing experimental results for each method. Finally, we conclude and discuss future research directions.

\section{Nested Dichotomies}

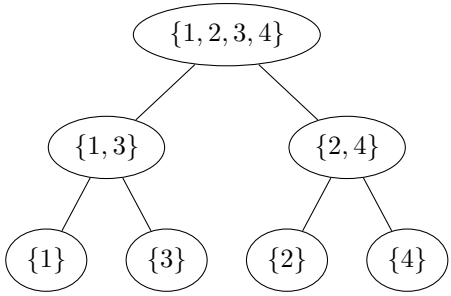
\begin{figure}[t]
	\centering
	\begin{tikzpicture}
		\tikzstyle{level 1}=[sibling distance = 32mm]
		\tikzstyle{level 2}=[sibling distance = 16mm]
		\usetikzlibrary{shapes}
		\node[ellipse,draw](z){$\{1,2,3,4\}$}
		  child
		  {
		  	node[ellipse,draw]{$\{1,3\}$}
		  	child
		  	{
		  		node[ellipse,draw]{$\{1\}$}
		  	}
		  	child
		  	{
		  		node[ellipse,draw]{$\{3\}$}
		  	}
		  }
		  child
		  {
	    		node[ellipse,draw]{$\{2,4\}$} 
	    		child
	    		{
	    			node[ellipse,draw] {$\{2\}$}
	    		} 
	    		child
	    		{
	    			node[ellipse,draw] {$\{4\}$}
	    		}
	    	  };
	\end{tikzpicture}
	
	\caption{\label{fig:nd_example} An example of nested a dichotomy for a four class problem.}
\end{figure}

Nested dichotomies are used as a binary decomposition method for multiclass problems~\cite{frank2004ensembles}. In this paper, we only consider the case where a single nested dichotomy structure is built, although generally superior performance can be achieved by training an ensemble of nested dichotomies with different structures~\cite{rodriguez2010forests}. Ensembles of nested dichotomies have been shown to outperform binary decomposition methods like one-vs-all~\cite{rifkin2004defense}, one-vs-one~\cite{friedman1996another} and error-correcting output codes~\cite{dietterich1995solving}, on some classification problems~\cite{frank2004ensembles}. 

The structure of a nested dichotomy can have a large impact on the predictive performance, training time and prediction time. To this end, several methods have been proposed for deciding the structure of nested dichotomies~\cite{dong2005ensembles,leathart2016building,melnikov2018effectiveness,wever2018ensembles,leathart2018ensembles}. In this paper, we focus on a simple method that randomly splits the set of classes into two at each internal node.

As previously stated, a useful feature of nested dichotomies is the ability to produce multiclass probability estimates $\mathbf{\hat{p}}_i$ for a test instance $(\mathbf{x}_i, y_i)$ from the product of binary estimates on the path $\mathcal{P}_c$ to the leaf node corresponding to class $c$: 
\begin{align*}
		\mathbf{\hat{p}}^{(c)}_i &= p(y_i = c | \mathbf{x}_i) \\ 
	\begin{split}
			&= \prod_{k \in \mathcal{P}_c} \big( \mathbb{I}(c \in \mathcal{C}_{k1}) p(c \in \mathcal{C}_{k1}|\mathbf{x}_i, y_i \in \mathcal{C}_{k}) ~ + \\
			& \qquad \quad ~~\mathbb{I}(c \in \mathcal{C}_{k2}) p(c \in \mathcal{C}_{k2}|\mathbf{x}_i, y_i \in \mathcal{C}_{k}) \big).
	\end{split}
\end{align*}
where $\mathbb{I}(\cdot)$ is the indicator function, $\mathcal{C}_k$ is the set of classes present at node $k$ and $\mathcal{C}_{k1}, \mathcal{C}_{k2} \subset \mathcal{C}_k$ are the sets of classes present at the left and right child of node $k$, respectively. If desired, one can still perform greedy inference by taking the most promising branch at each split point, but this is not guaranteed to find the best solution~\cite{beygelzimer2009error}. Having binary class probability estimates facilitates efficient tree search techniques~\cite{kumar2013beam,mena2015using,dembczynski2016consistency} for better inference, as well as top-$k$ prediction. 

\section{Probability Calibration}
Probability calibration is the task of transforming the outputs of a classifier to accurate probabilities. It is useful in a range of settings, such as cost-sensitive classification and scenarios where the outputs of a model are used as inputs for another. It is also important in real world decision making systems to know when a prediction from a model is likely to be incorrect. For example, models used in an automated healthcare system should indicate when they are unsure of a prediction so a doctor can take control. 

Some models, like logistic regression, tend to be well calibrated out-of-the-box, while other models like na\"ive Bayes and boosted decision trees usually exhibit poor calibration, despite high classification accuracy~\cite{niculescu2005predicting}. Some other models such as support vector machines cannot output probabilities at all, so calibration can be applied to produce a probability estimate. Calibration is typically applied as a post-processing step---a calibration model is trained to transform the output score from a model into a well calibrated probability. The calibration model should be trained on a held-out dataset that was not used for training of the base model to avoid overfitting.

\subsection{Calibration Methods}
The most commonly used calibration methods in practice are Platt scaling~\cite{platt1999probabilistic} and isotonic regression~\cite{zadrozny2001obtaining}. Both of these methods are only directly applicable to binary problems, but standard multiclass transformation techniques can be used to apply them to multiclass problems~\cite{zadrozny2002transforming}. 

\subsubsection{Platt Scaling} is a technique that fits a sigmoid curve 
\begin{align*}
	\boldsymbol{\sigma}(z_i) = \frac{1}{1 + e^{\alpha z_i + \beta}}
\end{align*}
from the output of a binary classifier $z_i$ to the true labels. The parameters $\alpha$ and $\beta$ are fitted using logistic regression. Platt scaling was originally proposed for scaling the output of SVMs, but has been shown to be an effective calibration technique for a range of models~\cite{niculescu2005predicting}. Usually, Platt scaling is applied to the log-odds (sometimes called logits) of the positive class, rather than the probability.

\subsubsection{Matrix and Vector Scaling}
are simple extensions of Platt scaling for multiclass problems~\cite{guo2017calibration}. In matrix scaling, instead of single parameters $\alpha$ and $\beta$, a matrix $\mathbf{W}$ and bias vector $\mathbf{b}$ is learned:
\begin{align*}
		\boldsymbol{\sigma}(\mathbf{z}_i) = \frac{1}{1 + e^{\mathbf{Wz}_i + \mathbf{b}}}
\end{align*}
where $\mathbf{z}_i$ is the vector of the log-odds of each class for instance $i$. Matrix scaling is equivalent to a standard multiple logistic regression model applied to the log-odds. It is expensive for datasets with many classes, as the weight matrix~$\mathbf{W}$ grows quadratically with the number of classes. Vector scaling is proposed to overcome this. It is a variant where~$\mathbf{W}$ is restricted to be a diagonal matrix to achieve scaling that is linear in the number of classes.

\subsubsection{Isotonic Regression} is a non-parametric technique for probability calibration~\cite{zadrozny2001obtaining}. It fits a piecewise constant function to minimise the mean squared error between the estimated class probabilities and the true labels. Isotonic regression is a more general method than Platt scaling because no assumptions are made about the function used to map classifier outputs, other than that the function is non-decreasing (isotonicity). Isotonic regression has been found to work well as a calibration model, but the flexibility of the fitted function means it can overfit on small samples.

\subsubsection{Other related work.}
In this paper, efficiency is a concern as there are many models to be calibrated. For this reason, we opt for the simpler calibration methods mentioned above in our experiments. However, there are several more expressive (and expensive) calibration methods in the literature. \citeauthor{jiang2011smooth} (\citeyear{jiang2011smooth}) and \citeauthor{zhong2013accurate} (\citeyear{zhong2013accurate}) propose methods for creating a smooth spline from the piecewise constant function produced in isotonic regression. \citeauthor{naeini2015obtaining} (\citeyear{naeini2015obtaining}) propose a method for performing Bayesian averaging over all possible binning schemes---schemes that split the probability space into several bins and calibrated probability value is established per bin. \citeauthor{leathart2017probability} (\citeyear{leathart2017probability}) split the feature space into regions using a decision tree and build a localised calibration model in each region.

\subsection{Measuring Miscalibration}

The level of probability calibration that a model exhibits is commonly measured by the negative log-likelihood (NLL):
\begin{align*}
	\text{NLL} = - \frac{1}{n} \sum_{i=1}^n \mathbf{y}_i \log \mathbf{\hat{p}}_i
\end{align*}
where $n$ is the number of examples, $\mathbf{y}_i$ is the one-hot true label for an instance and $\mathbf{\hat{p}}_i$ is the estimated probability distribution. NLL heavily penalises probability estimates that are far from the true label. For this reason, models which optimise NLL in training tend to be well calibrated, although interestingly it has been shown recently that the kinds of architectures used in modern neural networks can also produce poorly calibrated models~\cite{guo2017calibration}. NLL is also commonly used as a general measure of model accuracy.

\pgfplotscreateplotcyclelist{growthlist}{%
black, solid, every mark/.append style={solid, fill=black}, mark=square*\\%
black, solid, every mark/.append style={solid, fill=black}, mark=*\\%
black, solid, every mark/.append style={solid, fill=black}, mark=triangle*\\%
}

\pgfplotstableread[col sep=tab,header=true]{
	x	y
	0.01	0.125
	0.015	0.31
	0.02	0.49
	0.07	0.62
	0.22	0.72
	0.57	0.74
	0.91	0.87
	0.99	0.93
}\nbcalibration
\pgfplotstableread[col sep=tab,header=true]{
	x	y
	0.03	0.08
	0.04	0.03
	0.05	0.045
	0.08	0.12
	0.10	0.13
	0.13	0.105
	0.235	0.27
	0.40	0.45
	0.625	0.625
	0.74	0.64
	0.77	0.77
	0.81	0.85
	0.86	0.82
	0.90	0.9
}\nbisotoniccalibration
\pgfplotstableread[col sep=tab,header=true]{
	x	y
	0.10	0.09
	0.12	0.06
	0.14	0.09
	0.16	0.10
	0.18	0.16
	0.22	0.06
	0.28	0.23
	0.34	0.44
	0.39	0.64
	0.50	0.65
	0.61	0.77
	0.76	0.8
	0.87	0.81
	0.95	0.87
}\nbplattcalibration

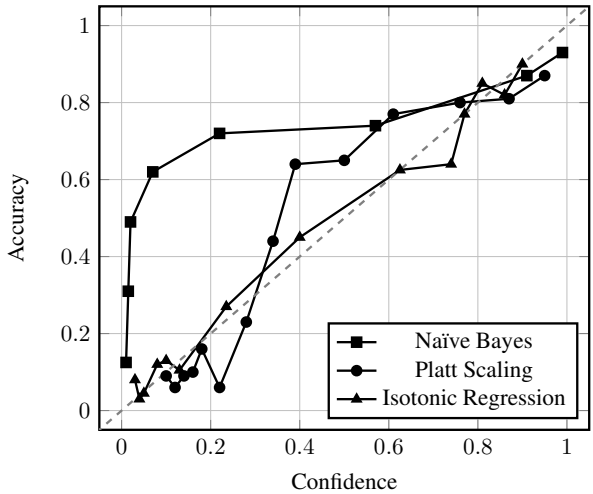
\begin{figure}[t]
	\resizebox{0.45\textwidth}{!}{	
		\begin{tikzpicture}     
	        \begin{axis}[    
			        cycle list name=growthlist,
		            width=9cm,
		            height=8cm,
		            line width=1pt,                               
		            ymin=-0.05,
		            ymax=1.05,
		            ytick={0,0.2,0.4,0.6,0.8,1.0},
		            xtick={0,0.2,0.4,0.6,0.8,1.0},
		            xmin=-0.05,
		            xmax=1.05,
		            grid=major,
		            ylabel={Accuracy},
		            xlabel={Confidence},
		            legend style={at={(0.978, 0.025)}, anchor=south east}
	            ]       
		        \addplot table [y=y, x=x] {\nbcalibration};   
		        \addplot table [y=y, x=x] {\nbplattcalibration}; 
		        \addplot table [y=y, x=x] {\nbisotoniccalibration}; 
		        \addplot [gray, dashed] {x};
		        \legend{Na\"ive Bayes, Platt Scaling, Isotonic Regression}
	        \end{axis}
	    \end{tikzpicture}
	}
	\caption{Reliability diagram for of the \texttt{credit} dataset from the UCI repository. Curves are shown for the predictions made by na\"ive Bayes, as well as calibrated probabilities through Platt scaling and isotonic regression.}
	\label{fig:reliability_example}
\end{figure}

Probability calibration for binary classification tasks can be visualised through reliability diagrams~\cite{degroot1983comparison}. In reliability diagrams, the probability range $[0,1]$ is discretised into $K$ bins $B_1, \dots, B_k$. These bins are chosen such that they have equal width, or equal numbers of examples. The \textit{confidence} of each bin is given as the average prediction of examples that fall inside the bin, while the \textit{accuracy} of each bin is the empirical accuracy:
\begin{align*}
	\text{acc}(B_k) &= \frac{1}{|B_k|} \sum_{i \in B_k} \mathbb{I}(\hat{y}_i = y_i), \\
	\text{conf}(B_k) &= \frac{1}{|B_k|} \sum_{i \in B_k} \hat{p}_i
\end{align*}
where $y_i$ is the true binary label, $\hat{y}_i$ is the predicted label, and $\hat{p}_i$ is the predicted probability for an instance $i$. Intuitively, a well calibrated classifier should have comparable confidence and accuracy for each bin. The confidence and accuracy are plotted against each other for each bin, producing a straight diagonal line for a well calibrated classifier (Fig.~\ref{fig:reliability_example}).

\citeauthor{naeini2015obtaining} (\citeyear{naeini2015obtaining}) applied the same idea to give a direct quantitative measure of calibration, called the \textit{expected calibration error} (ECE):
\begin{align*}
	\text{ECE} = \sum_{k=1}^K \frac{|B_k|}{n}~\Big| \text{acc}(B_k) - \text{conf}(B_k) \Big|.
\end{align*}
This is simply a weighted average of the residuals in a reliability diagram, weighted by the number of instances that fall inside each bin.

\section{Internal Calibration}
In this section, we investigate the effect of calibrating the internal models of nested dichotomies. Our hypothesis is that by improving the quality of the binary probability estimates, the final multiclass predictive performance will be improved. This is due to the fact that multiclass probability estimates are produced by computing the product of a series of binary probability estimates. If the binary probability estimates are not well calibrated, then these errors will accumulate throughout the calculation.

\subsection{Theoretical Motivation}
It seems reasonable that improving the calibration of internal models will result in superior probability estimates for the nested dichotomy, but can we theoretically quantify this improvement? It turns out that reducing the binary NLL of any internal model by some amount $\delta$ strictly reduces the multiclass NLL of the nested dichotomy, and depending on the depth of the internal model being calibrated, the reduction in multiclass NLL can be as high as $\delta$.

\begin{proposition}
	The NLL of an instance under a nested dichotomy is equal to the sum of NLLs of the instance under the binary models on the path from the root node to the leaf node.	
\end{proposition}
\begin{proof}
	The NLL of an instance is given by
	\begin{align}
		\label{eq:nll}
		\text{NLL} = - \mathbf{y}_i \log \mathbf{\hat{p}}_i = - \log \mathbf{\hat{p}}_i^{(c)}
	\end{align}
	where $\mathbf{\hat{p}}_i^{(c)}$ is the probability estimate for the true class $c$. Let $\mathcal{P}_c$ be the set of internal nodes on the path from the root to the leaf corresponding to class $c$. Then, $\mathbf{\hat{p}}_i^{(c)}$ can be expressed as 
	\begin{align}
		\mathbf{\hat{p}}_i^{(c)} = \prod_{k \in \mathcal{P}_c} \tilde{y}_{ik}\hat{p}_{ik} + (1-\tilde{y}_{ik})(1-\hat{p}_{ik})
	\end{align}
	where $\hat{p}_{ik} \in [0,1]$ is the scalar prediction and $\tilde{y}_{ik} \in \{0,1\}$ is the meta-label for instance $i$ for the binary model at node $k$ respectively. Because $\tilde{y}_{ik} \in \{0,1\}$, it is equivalent to write
	\begin{align}
		\mathbf{\hat{p}}_i^{(c)} = \prod_{k \in \mathcal{P}_c} \hat{p}_{ik}^{~\tilde{y}_{ik}}(1-\hat{p}_{ik})^{(1-\tilde{y}_{ik})}.
	\end{align}
	Plugging this into (\ref{eq:nll}) yields
	\begin{align}
		\text{NLL} &= - \log \prod_{k \in \mathcal{P}_c} \hat{p}_{ik}^{~\tilde{y}_{ik}}(1-\hat{p}_{ik})^{(1-\tilde{y}_{ik})}\\
					&= - \sum_{k \in \mathcal{P}_c} \log \Big( \hat{p}_{ik}^{~\tilde{y}_{ik}}(1-\hat{p}_{ik})^{(1-\tilde{y}_{ik})}\Big) \\
					&= - \sum_{k \in \mathcal{P}_c} \tilde{y}_{ik} \log \hat{p}_{ik} + (1-\tilde{y}_{ik}) \log (1-\hat{p}_{ik}),
	\end{align}
	the sum of NLLs from the models $k \in \mathcal{P}_c$.
\end{proof}
It directly follows that reducing the binary NLL for the model at internal node~$k$ by some amount~$\delta$ results in a reduction of the multiclass NLL by $\delta$ for each class corresponding to the leaf nodes that are descendants of node~$k$. This means that a calibration resulting in a binary NLL reduction of~$\delta$ for some internal node~$k$ reduces the multiclass NLL by~$\delta (n_k / n)$, where~$n_k$ is the number of examples that belong to classes whose corresponding leaf nodes are descendants of~$k$. 

The number of descendant leaf nodes for an internal node~$k$ in a balanced tree equals~$2^l$, where~$l$ is the length of the paths to the leaf nodes from~$k$. Therefore, it is more important to ensure the models nearer the root node are well calibrated than the models nearer leaf nodes. However, calibration of entire layers of a nested dichotomy should have an effect that is independent of the particular layer being calibrated. Note that even though each layer has twice as many internal models as the layer before it, each internal model at the lower layer is trained on approximately half the amount of data as the models in the previous layer. Platt scaling and isotonic regression, the internal calibration models we consider in our experiments, both have linear complexity in the number of examples, so the time taken to train calibration models for each layer of a nested dichotomy should be comparable.

\section{External Calibration}
\pgfplotstableread[col sep=tab,header=true]{
x	y
0.0138888888889	0.0191474221044
0.107142857143	0.0768255550544
0.10843373494	0.124979576206
0.107594936709	0.176925540017
0.221649484536	0.227344941908
0.245	0.274894744713
0.300970873786	0.325205514363
0.340425531915	0.378164445044
0.421052631579	0.425833188011
0.451187335092	0.474712801986
0.548812664908	0.525287198014
0.578947368421	0.574166811989
0.659574468085	0.621835554956
0.699029126214	0.674794485637
0.755	0.725105255287
0.778350515464	0.772655058092
0.892405063291	0.823074459983
0.89156626506	0.875020423794
0.892857142857	0.923174444946
0.986111111111	0.980852577896
}\aloione

\pgfplotstableread[col sep=tab,header=true]{
x	y
0.0102040816327	0.0223532369421
0.0462427745665	0.0736865724452
0.0960512273212	0.125323901744
0.154279279279	0.174684334384
0.216743119266	0.224867966989
0.246268656716	0.274050391249
0.313901345291	0.323567842625
0.364583333333	0.373724314886
0.434782608696	0.423039352316
0.598006644518	0.472384798762
0.66	0.52480770126
0.622950819672	0.574352740959
0.69375	0.625379590694
0.769911504425	0.674152414018
0.797752808989	0.726031125756
0.871428571429	0.772798650058
0.947368421053	0.822244274696
1.0	0.873511478841
0.958333333333	0.927149495506
1.0	0.962564315126
}\aloitwo

\pgfplotstableread[col sep=tab,header=true]{
x	y
0.009656004828	0.0206368603459
0.0475314320761	0.0728549953626
0.0966139954853	0.123212528152
0.152187698161	0.173826191355
0.233802816901	0.223968263358
0.308005427408	0.273273714456
0.404896421846	0.324518147657
0.471014492754	0.37426811207
0.52	0.425415446356
0.576687116564	0.474519307201
0.712962962963	0.52272401725
0.765957446809	0.571484140789
0.872093023256	0.626569295673
0.961538461538	0.67458802985
0.866666666667	0.724705921367
0.928571428571	0.773818116669
1.0	0.827243147674
1.0	0.873179646065
1.0	0.917183694904
1.0	0.965474041745
}\aloithree

\pgfplotstableread[col sep=tab,header=true]{
x	y
0.00466556204761	0.0145126974086
0.0397140587768	0.0715193766989
0.118998768978	0.12216633025
0.19023689878	0.173486844793
0.299441340782	0.22419780792
0.406956521739	0.273666045116
0.461009174312	0.323547806005
0.517647058824	0.373174402019
0.572463768116	0.424369188789
0.675	0.474096472539
0.786516853933	0.523772481922
0.875	0.571343918034
0.866666666667	0.626727386425
0.947368421053	0.672752545585
1.0	0.722140276084
0.952380952381	0.776117850958
1.0	0.829524254704
1.0	0.874751751085
1.0	0.909331279949
}\aloifour

\pgfplotstableread[col sep=tab,header=true]{
x	y
0.00241114329824	0.00975511591901
0.045175588135	0.0708129993205
0.141336739038	0.121868718501
0.227902946274	0.172172233192
0.394444444444	0.222706963243
0.485193621868	0.272605304346
0.548286604361	0.321435311625
0.598837209302	0.371610729887
0.660194174757	0.426087155106
0.824175824176	0.472262350835
0.910447761194	0.526570557137
0.923076923077	0.574960504151
0.931034482759	0.625180060084
1.0	0.676967343578
1.0	0.719663799177
0.933333333333	0.771641378792
1.0	0.829764346809
1.0	0.873026535545
1.0	0.904970697338
}\aloifive

\pgfplotstableread[col sep=tab,header=true]{
x	y
0.00129198966408	0.00603459139884
0.0509871244635	0.0701003050992
0.160559305689	0.121075497722
0.306709265176	0.172874331143
0.462633451957	0.22321236465
0.597826086957	0.273190910834
0.576763485477	0.32155926943
0.734848484848	0.371461626329
0.84693877551	0.424518408774
0.910447761194	0.472648130956
0.963636363636	0.522501987747
0.97619047619	0.575929333832
0.943396226415	0.62524727833
1.0	0.670894495118
1.0	0.725854867196
0.916666666667	0.770277679809
1.0	0.829427580548
1.0	0.872093309329
}\aloisix

\pgfplotstableread[col sep=tab,header=true]{
x	y
0.0321782178218	0.023565655571
0.104838709677	0.0729313565863
0.192307692308	0.125971518841
0.192893401015	0.175092358615
0.166666666667	0.224906730722
0.194444444444	0.276560561075
0.279761904762	0.326366540933
0.358695652174	0.37393454474
0.325714285714	0.423611234171
0.454022988506	0.473238671631
0.545977011494	0.526761328369
0.674285714286	0.576388765829
0.641304347826	0.62606545526
0.720238095238	0.673633459067
0.805555555556	0.723439438925
0.833333333333	0.775093269278
0.807106598985	0.824907641385
0.807692307692	0.874028481159
0.895161290323	0.927068643414
0.967821782178	0.976434344429
}\aloicalone

\pgfplotstableread[col sep=tab,header=true]{
x	y
0.0173215040135	0.0187420846122
0.0909090909091	0.0733493125827
0.122399020808	0.123046062662
0.178512396694	0.174923230735
0.240246406571	0.224943776353
0.302428256071	0.273503567769
0.301675977654	0.325864940319
0.346570397112	0.376641576977
0.456928838951	0.424466468592
0.489177489177	0.47313224411
0.466981132075	0.526864137143
0.529702970297	0.574152042423
0.602409638554	0.626272711585
0.717647058824	0.674896521963
0.684563758389	0.722557216086
0.781690140845	0.774433702041
0.757009345794	0.825496211796
0.867768595041	0.874174476672
0.897058823529	0.927547814997
0.975460122699	0.976806901216
}\aloicaltwo

\pgfplotstableread[col sep=tab,header=true]{
x	y
0.00851063829787	0.00906818864045
0.0921866216976	0.0726115979305
0.113947128532	0.122992135706
0.209039548023	0.173132645769
0.209979209979	0.224079562967
0.30843373494	0.273644539507
0.312714776632	0.325277416902
0.32119205298	0.374211215362
0.421052631579	0.423689393516
0.442105263158	0.472636659709
0.5	0.523389283491
0.580808080808	0.577118323881
0.613793103448	0.624268300625
0.664383561644	0.672705854092
0.688888888889	0.726670990435
0.759259259259	0.77602225109
0.80198019802	0.825787377123
0.87962962963	0.877266926731
0.915254237288	0.926702193341
0.928934010152	0.977714596555
}\aloicalthree

\pgfplotstableread[col sep=tab,header=true]{
x	y
0.00663761860162	0.00611911767163
0.0759193357058	0.0723511464525
0.111282843895	0.122017780741
0.188861985472	0.173373113546
0.200750469043	0.223572741593
0.255369928401	0.274841351809
0.281818181818	0.325038435403
0.375	0.37631223294
0.446808510638	0.424620935975
0.566371681416	0.472184999995
0.513333333333	0.523354726834
0.624060150376	0.572546840122
0.639705882353	0.621967183205
0.709401709402	0.675192050854
0.704081632653	0.725325892069
0.754901960784	0.772923360478
0.753246753247	0.82422685456
0.860465116279	0.878172449585
0.866666666667	0.926928241432
0.94375	0.980254287271
}\aloicalfour

\pgfplotstableread[col sep=tab,header=true]{
x	y
0.00484188228174	0.00471673123846
0.0734966592428	0.07174753037
0.101859337106	0.122150326337
0.171983356449	0.172708244113
0.200764818356	0.223093020416
0.262529832936	0.274066518835
0.3125	0.323056153464
0.340996168582	0.374899322309
0.394495412844	0.425956618532
0.485549132948	0.474483187623
0.56875	0.522530890021
0.723076923077	0.57672667108
0.81954887218	0.623586338527
0.739130434783	0.671918686858
0.747572815534	0.724837478144
0.835051546392	0.775914629566
0.776119402985	0.82880231078
0.88	0.876724723724
0.879120879121	0.926432358209
0.928571428571	0.984976424133
}\aloicalfive

\pgfplotstableread[col sep=tab,header=true]{
x	y
0.00337650494769	0.00326799141466
0.0622627182992	0.0717744920332
0.101652892562	0.122356906973
0.16231884058	0.172593896684
0.218295218295	0.223871223156
0.311239193084	0.273750134963
0.305369127517	0.322888293573
0.381395348837	0.374029443273
0.451776649746	0.423875232996
0.448051948052	0.473002357653
0.604651162791	0.524871017032
0.619402985075	0.57470153669
0.7	0.624661358929
0.754385964912	0.675169738979
0.844660194175	0.726740604337
0.741935483871	0.775844630003
0.83908045977	0.823279551993
0.896103896104	0.875534148775
0.907216494845	0.927426321652
0.966804979253	0.981509635828
}\aloicalsix

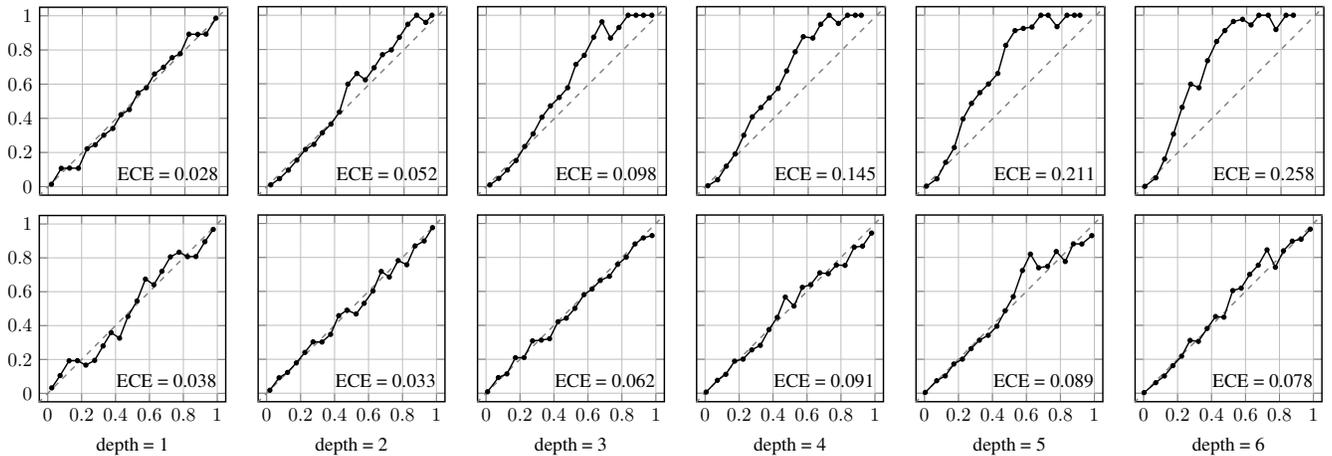
\begin{figure*}[t]
	\centering
	\resizebox{0.177\textwidth}{!}{	
		\begin{tikzpicture}     
	        \begin{axis}[    
			        cycle list name=growthlist,
		            width=5cm,
		            height=5cm,
		            line width=0.75pt, 
		            mark size=1pt,                              
		            ymin=-0.05,
		            ymax=1.05,
		            ytick={0,0.2,0.4,0.6,0.8,1.0},
		            xtick={0,0.2,0.4,0.6,0.8,1.0},
		            xticklabels={,,},
		            xmin=-0.05,
		            xmax=1.05,
		            grid=major,
	            ]       
		        \addplot [gray, dashed] {x};
		        \addplot table [y=x, x=y] {\aloione};   
		        \node[] at (axis cs: .7,.075) {ECE = 0.028};
	        \end{axis}
	    \end{tikzpicture}
	}
	\resizebox{0.159\textwidth}{!}{	
		\begin{tikzpicture}     
	        \begin{axis}[    
			        cycle list name=growthlist,
		            width=5cm,
		            height=5cm,
		            line width=0.75pt, 
		            mark size=1pt,                              
		            ymin=-0.05,
		            ymax=1.05,
		            ytick={0,0.2,0.4,0.6,0.8,1.0},
		            xtick={0,0.2,0.4,0.6,0.8,1.0},
		            yticklabels={,,},
		            xticklabels={,,},
		            xmin=-0.05,
		            xmax=1.05,
		            grid=major,
	            ]       
		        \addplot [gray, dashed] {x};
		        \addplot table [y=x, x=y] {\aloitwo};   
		        \node[] at (axis cs: .7,.075) {ECE = 0.052};
	        \end{axis}
	    \end{tikzpicture}
	}
	\resizebox{0.159\textwidth}{!}{	
		\begin{tikzpicture}     
	        \begin{axis}[    
			        cycle list name=growthlist,
		            width=5cm,
		            height=5cm,
		            line width=0.75pt,
		            mark size=1pt,                               
		            ymin=-0.05,
		            ymax=1.05,
		            ytick={0,0.2,0.4,0.6,0.8,1.0},
		            xtick={0,0.2,0.4,0.6,0.8,1.0},
		            yticklabels={,,},
		            xticklabels={,,},
		            xmin=-0.05,
		            xmax=1.05,
		            grid=major,
	            ]       
		        \addplot [gray, dashed] {x};
		        \addplot table [y=x, x=y] {\aloithree};   
		        \node[] at (axis cs: .7,.075) {ECE = 0.098};
	        \end{axis}
	    \end{tikzpicture}
	}
	\resizebox{0.159\textwidth}{!}{	
		\begin{tikzpicture}     
	        \begin{axis}[    
			        cycle list name=growthlist,
		            width=5cm,
		            height=5cm,
		            line width=0.75pt,
		            mark size=1pt,                               
		            ymin=-0.05,
		            ymax=1.05,
		            ytick={0,0.2,0.4,0.6,0.8,1.0},
		            xtick={0,0.2,0.4,0.6,0.8,1.0},
		            yticklabels={,,},
		            xticklabels={,,},
		            xmin=-0.05,
		            xmax=1.05,
		            grid=major,
	            ]       
		        \addplot [gray, dashed] {x};
		        \addplot table [y=x, x=y] {\aloifour};   
		        \node[] at (axis cs: .7,.075) {ECE = 0.145};
	        \end{axis}
	    \end{tikzpicture}
	}
	\resizebox{0.159\textwidth}{!}{	
		\begin{tikzpicture}     
	        \begin{axis}[    
			        cycle list name=growthlist,
		            width=5cm,
		            height=5cm,
		            line width=0.75pt,  
		            mark size=1pt,                             
		            ymin=-0.05,
		            ymax=1.05,
		            ytick={0,0.2,0.4,0.6,0.8,1.0},
		            xtick={0,0.2,0.4,0.6,0.8,1.0},
		            yticklabels={,,},
		            xticklabels={,,},
		            xmin=-0.05,
		            xmax=1.05,
		            grid=major,
	            ]       
		        \addplot [gray, dashed] {x};
		        \addplot table [y=x, x=y] {\aloifive};   
		        \node[] at (axis cs: .7,.075) {ECE = 0.211};
	        \end{axis}
	    \end{tikzpicture}
	}
	\resizebox{0.159\textwidth}{!}{	
		\begin{tikzpicture}     
	        \begin{axis}[    
			        cycle list name=growthlist,
		            width=5cm,
		            height=5cm,
		            line width=0.75pt,                               
		            mark size=1pt,
		            ymin=-0.05,
		            ymax=1.05,
		            ytick={0,0.2,0.4,0.6,0.8,1.0},
		            xtick={0,0.2,0.4,0.6,0.8,1.0},
		            yticklabels={,,},
		            xticklabels={,,},
		            xmin=-0.05,
		            xmax=1.05,
		            grid=major,
	            ]       
		        \addplot [gray, dashed] {x};
		        \addplot table [y=x, x=y] {\aloisix};   
		        \node[] at (axis cs: .7,.075) {ECE = 0.258};
	        \end{axis}
	    \end{tikzpicture}
	}
	 \\
	\resizebox{0.177\textwidth}{!}{	
		\begin{tikzpicture}     
	        \begin{axis}[    
			        cycle list name=growthlist,
		            width=5cm,
		            height=5cm,
		            line width=0.75pt, 
		            mark size=1pt,                              
		            ymin=-0.05,
		            ymax=1.05,
		            ytick={0,0.2,0.4,0.6,0.8,1.0},
		            xtick={0,0.2,0.4,0.6,0.8,1.0},
		            xmin=-0.05,
		            xmax=1.05,
		            grid=major,
		            xlabel={depth = 1},
	            ]       
		        \addplot [gray, dashed] {x};
		        \addplot table [y=x, x=y] {\aloicalone};   
		        \node[] at (axis cs: .7,.075) {ECE = 0.038};
	        \end{axis}
	    \end{tikzpicture}
	}
	\resizebox{0.159\textwidth}{!}{	
		\begin{tikzpicture}     
	        \begin{axis}[    
			        cycle list name=growthlist,
		            width=5cm,
		            height=5cm,
		            line width=0.75pt, 
		            mark size=1pt,                              
		            ymin=-0.05,
		            ymax=1.05,
		            ytick={0,0.2,0.4,0.6,0.8,1.0},
		            xtick={0,0.2,0.4,0.6,0.8,1.0},
		            yticklabels={,,},
		            xmin=-0.05,
		            xmax=1.05,
		            grid=major,
		            xlabel={depth = 2},
	            ]       
		        \addplot [gray, dashed] {x};
		        \addplot table [y=x, x=y] {\aloicaltwo};   
		        \node[] at (axis cs: .7,.075) {ECE = 0.033};
	        \end{axis}
	    \end{tikzpicture}
	}
	\resizebox{0.159\textwidth}{!}{	
		\begin{tikzpicture}     
	        \begin{axis}[    
			        cycle list name=growthlist,
		            width=5cm,
		            height=5cm,
		            line width=0.75pt,
		            mark size=1pt,                               
		            ymin=-0.05,
		            ymax=1.05,
		            ytick={0,0.2,0.4,0.6,0.8,1.0},
		            xtick={0,0.2,0.4,0.6,0.8,1.0},
		            yticklabels={,,},
		            xmin=-0.05,
		            xmax=1.05,
		            grid=major,
		            xlabel={depth = 3},
	            ]       
		        \addplot [gray, dashed] {x};
		        \addplot table [y=x, x=y] {\aloicalthree};   
		        \node[] at (axis cs: .7,.075) {ECE = 0.062};
	        \end{axis}
	    \end{tikzpicture}
	}
	\resizebox{0.159\textwidth}{!}{	
		\begin{tikzpicture}     
	        \begin{axis}[    
			        cycle list name=growthlist,
		            width=5cm,
		            height=5cm,
		            line width=0.75pt,
		            mark size=1pt,                               
		            ymin=-0.05,
		            ymax=1.05,
		            ytick={0,0.2,0.4,0.6,0.8,1.0},
		            xtick={0,0.2,0.4,0.6,0.8,1.0},
		            yticklabels={,,},
		            xmin=-0.05,
		            xmax=1.05,
		            grid=major,
		            xlabel={depth = 4},
	            ]       
		        \addplot [gray, dashed] {x};
		        \addplot table [y=x, x=y] {\aloicalfour};   
		        \node[] at (axis cs: .7,.075) {ECE = 0.091};
	        \end{axis}
	    \end{tikzpicture}
	}
	\resizebox{0.159\textwidth}{!}{	
		\begin{tikzpicture}     
	        \begin{axis}[    
			        cycle list name=growthlist,
		            width=5cm,
		            height=5cm,
		            line width=0.75pt,  
		            mark size=1pt,                             
		            ymin=-0.05,
		            ymax=1.05,
		            ytick={0,0.2,0.4,0.6,0.8,1.0},
		            xtick={0,0.2,0.4,0.6,0.8,1.0},
		            yticklabels={,,},
		            xmin=-0.05,
		            xmax=1.05,
		            grid=major,
				    xlabel={depth = 5}
	            ]       
		        \addplot [gray, dashed] {x};
		        \addplot table [y=x, x=y] {\aloicalfive};   
		        \node[] at (axis cs: .7,.075) {ECE = 0.089};
	        \end{axis}
	    \end{tikzpicture}
	}
	\resizebox{0.159\textwidth}{!}{	
		\begin{tikzpicture}     
	        \begin{axis}[    
			        cycle list name=growthlist,
		            width=5cm,
		            height=5cm,
		            line width=0.75pt,                               
		            mark size=1pt,
		            ymin=-0.05,
		            ymax=1.05,
		            ytick={0,0.2,0.4,0.6,0.8,1.0},
		            xtick={0,0.2,0.4,0.6,0.8,1.0},
		            yticklabels={,,},
		            xmin=-0.05,
		            xmax=1.05,
		            grid=major,
				    xlabel={depth = 6},
	            ]       
		        \addplot [gray, dashed] {x};
		        \addplot table [y=x, x=y] {\aloicalsix};   
		        \node[] at (axis cs: .7,.075) {ECE = 0.078};
	        \end{axis}
	    \end{tikzpicture}
	}	 
	\caption{Reliability plots for a nested dichotomy, cut off at increasing depth. The nested dichotomy has logistic regression base learners, which individually, are well calibrated. Top: no external calibration. As the depth increases, the nested dichotomy becomes increasingly under-confident because of the effect of multiplying probabilities together. Bottom: externally calibrated using vector scaling. \label{fig:depth_reliability}}
\end{figure*}

As well as calibrating each internal model, we also consider external calibration of the entire nested dichotomy. Even models like logistic regression are usually not perfectly calibrated in practice. We hypothesise that these minor miscalibrations accumulate as the nested dichotomy gets deeper, leading to overall more serious miscalibration in the overall model, which can be rectified by an external calibration model. Naturally, this effect is greater for problems with more classes, as the paths to leaf nodes will be longer. 

As an illustrative investigation into the effect of nested dichotomy depth on their calibration, we built a nested dichotomy with logistic regression base learners for the ALOI dataset (see Tab.~\ref{tab:datasets}). Figure~\ref{fig:depth_reliability} shows reliability plots for this nested dichotomy that has been ``cut-off'' at incrementally increasing depths. A test example is considered to be classified correctly at depth $d$ if its actual class is in the subset of classes $\mathcal{C}_k$ of the node $k$ with highest probability and maximum depth $d$. Limited to a depth of one, the nested dichotomy is simply a single binary logistic regression model which exhibits good calibration. However, as the depth cut-off limit increases, it is clear that the nested dichotomy becomes increasingly \textit{under-confident}, i.e., bins that have high accuracy often have low confidence (Fig.~\ref{fig:depth_reliability}, top row). This corresponds to the reliability curve sitting above the diagonal line. The ECE increases approximately linearly with the depth of the tree.

 This is adequately and efficiently compensated for by applying vector scaling (Fig.~\ref{fig:depth_reliability}, bottom row). Vector scaling exhibits low complexity in the number of classes---only two parameters per class---making it suitable for problems with many classes typically handled by nested dichotomies. For externally calibrated nested dichotomies, the ECE initially increases linearly with the depth of the tree (although for~$d > 1$, the ECE values are much lower than their uncalibrated counterparts). However, at $d=5$, the ECE levels off and even begins to decrease slightly.

\section{Experiments}
\begin{table}
	\centering
	\caption{Datasets used in our experiments.\label{tab:datasets}}
	\fontsize{9pt}{10pt} \selectfont
	\begin{threeparttable}
		\begin{tabular}{lccc}
			\toprule
			\bf Name & \bf Instances & \bf Features & \bf Classes \\
			\midrule
			optdigits\tnote{1} & 5,620 & 64 & 10\\
			micromass\tnote{1} & 571 & 1,301 & 20 \\
			letter\tnote{1} & 20,000 & 16 & 26 \\
			devanagari\tnote{1} & 92,000 & 1,000 & 46 \\
			\midrule
			RCV1\tnote{2} & 15,564/518,571 & 47,236 & 53 \\
			sector\tnote{2} & 6,412/3,207 & 55,197 & 105 \\
			ALOI\tnote{2} & 97,200/10,800 & 128 & 1,000 \\
			ILSVR2010\tnote{3} & 1,111,406/150,000 & 1,000 & 1,000 \\
			ODP-5K\tnote{4} & 361,488/180,744 & 422,712 & 5,000\\
			\bottomrule
		\end{tabular}
		\begin{tablenotes}
			\centering
			\item [1] UCI Repository~\cite{lichman2013uci}	
			\item [2] LIBSVM Repository~\cite{chang2011libsvm}
			\item [3] ImageNet~\cite{ILSVRC15}
			\item [4] ODP~\cite{bennett2009refined}
		\end{tablenotes}
	\end{threeparttable}
\end{table}
In this section, we present experimental results of calibration of nested dichotomies using different base classifiers on a series of datasets. The datasets we used in our experiments are listed in Table~\ref{tab:datasets}, and were chosen to span a range of numbers of classes. Optdigits, letter and devanagari~\cite{acharya2015deep} are character recognition datasets for digits, latin letters, and Devanagari script respectively. Micromass~\cite{mahe2014automatic} is for the identification of microorganisms from mass spectroscopy data. RCV1, sector and ODP-5K are text categorisation tasks, while ALOI and ILSVR2010 are object recognition tasks. We use the visual codewords representation for ILSVR2010, available from the ImageNet competition website.

In order to obtain performance estimates, we performed 10 times 10-fold cross-validation for the datasets from the UCI repository, while adopting the standard train/test splits for the larger datasets with a larger number of classes ($m>50$). The number of instances stated in Table~\ref{tab:datasets} for the larger datasets are split into number of training and test instances. Note that in each fold and run of 10 times 10-fold cross-validation, a different random nested dichotomy structure is constructed. In the case of the larger datasets, the average of 10 randomly constructed nested dichotomies is reported. Standard deviations are given in parentheses, and the best result per row appears in bold face. The original ODP dataset contains 105,000 classes---we took the subset of the most frequent 5,000 classes to create ODP-5K for the purposes of this investigation. We also reduce the dimensionality to 1,000 when evaluating the performance on boosted trees, by using a Gaussian random projection~\cite{bingham2001random}. 

\begin{table}[t]
	\centering
	\caption{NLL of nested dichotomies with logistic regression, before and after external calibration is applied.\label{tab:logistic_regression_nll}}
	\fontsize{9.5pt}{10.5pt} \selectfont
	\begin{tabular}{lcc}
		\toprule
		\bf Dataset & \bf Baseline & \bf External VS\\
		\midrule
		optdigits & 0.302 (0.07) & \bf 0.301 (0.08)\\
		micromass &  5.918 (1.83) & \bf 1.878 (0.51)\\
		letter &  1.502 (0.06) & \bf 1.435 (0.08)\\
		devanagari & 2.430 (0.11) & \bf 2.028 (0.05) \\
		\midrule
		RCV1 & 1.004 (0.02) & \bf 0.584 (0.01)\\
		sector &  2.858 (0.01) & \bf 1.248 (0.02)\\
		ALOI & 3.604 (0.02) & \bf 3.050 (0.03) \\
		ILSVR2010 & 6.442 (0.01) & \bf 5.780 (0.01) \\
		ODP-5K & 5.792 (0.01) & \bf 4.967 (0.01)\\
		\bottomrule
	\end{tabular}
\end{table}

\begin{table}[t]
	\centering
	\caption{Classification accuracy of nested dichotomies with logistic regression, before and after external calibration.\label{tab:logistic_regression_acc}}
	\fontsize{9.5pt}{10.5pt} \selectfont
	\begin{tabular}{lcc}
		\toprule
		\bf Dataset & \bf Baseline & \bf External VS\\
		\midrule
		optdigits &  0.905 (0.02) & \bf 0.906 (0.03)\\
		micromass & \bf 0.804 (0.06) &  0.772 (0.05)\\
		letter &  0.512 (0.03) &  \bf 0.536 (0.03)\\
		devanagari & 0.428 (0.02) & \bf 0.428 (0.02) \\
		\midrule
		RCV1 & 0.814 (0.01) & \bf 0.855 (0.00)\\
		sector &  0.848 (0.01) & \bf 0.867 (0.00)\\
		ALOI & 0.274 (0.01) & \bf 0.331 (0.01)\\
		ILSVR2010 & \bf 0.063 (0.00) & 0.053 (0.00)\\
		ODP-5K & 0.189 (0.00) & \bf 0.228 (0.00)\\
		\bottomrule
	\end{tabular}
\end{table}

We implemented vector scaling~\cite{guo2017calibration} and nested dichotomies in Python, and used the implementations of the base learners, isotonic regression and Platt scaling available in \texttt{scikit-learn}~\cite{pedregosa2011scikit}. 

\subsection{Well Calibrated Base Learners}
As shown in Figure~\ref{fig:depth_reliability}, overall calibration of nested dichotomies can degrade as the depth of the tree increases, even if the base learners are well calibrated. To further investigate the effects of nested dichotomy depth on predictive performance, we performed experiments to determine the extent to which the classification accuracy and NLL are affected as well.

\begin{table*}[!h]
	\centering
	\caption{NLL of nested dichotomies with poorly calibrated base classifiers.\label{tab:nll_poor}}
	\fontsize{9.5pt}{10.5pt} \selectfont
	\begin{tabular}{llcccccc}
		\toprule
		\bf Base Model & \bf Dataset & \bf Baseline & \bf External VS & \bf Internal PS & \bf Both PS & \bf Internal IR & \bf Both IR \\
		\midrule
		\multirow{9}{*}{Na\"ive Bayes} 
			&  optdigits  &   4.252 (0.92)  &   0.841 (0.13)  &   0.852 (0.11)  &   0.807 (0.09)  &   0.714 (0.12)  &  \bf 0.642 (0.09)\\
			 &  micromass  &   8.668 (1.72)  &   1.624 (0.42)  &   0.926 (0.08)  &   0.752 (0.12)  &   0.766 (0.12)  &  \bf 0.712 (0.15)\\
			 &  letter  &   2.338 (0.08)  &   2.155 (0.08)  &   2.165 (0.06)  &   2.068 (0.07)  &   2.055 (0.07)  &  \bf 1.953 (0.06)\\
			 &  devanagari  &   13.14 (0.59)  &   3.310 (0.16)  &   2.986 (0.05)  &   2.602 (0.02)  &   2.754 (0.07)  &  \bf 2.444 (0.05) \\
			 \cmidrule{2-8}
			 &  RCV1  &   1.690 (0.19)  &  \bf 0.866 (0.01)  &   1.145 (0.07)  &   0.947 (0.03)  &   0.998 (0.04)  &   0.914 (0.02)\\
			 &  sector		 &   3.795 (0.36)  &   \bf 1.408 (0.09)  &   2.075 (0.20)  &   1.518 (0.10)  &   1.913 (0.21)  &   1.774 (0.09)\\
			 &  ALOI 			 &  32.98 (0.43)  &   6.841 (0.02)  &   5.533 (0.02)  &  4.333 (0.03)  &   4.859 (0.03)  &   \bf 4.137 (0.01) \\
			 &  ILSVR2010  &   32.39 (0.20)  &   6.819 (0.00)  &   6.162 (0.00)  &  6.125 (0.01)  &   6.176 (0.00)  &  \bf 6.116 (0.00)\\
			 &  ODP-5K  &  8.498 (0.31)  &  \bf 5.161 (0.05)  &  6.103 (0.00)  &  5.518 (0.02)  &  6.051 (0.11)  &  5.364 (0.01) \\
		\midrule
   		\multirow{9}{*}{Boosted Trees}
			 &  optdigits  &   3.861 (0.57)  &   0.634 (0.07)  &   0.403 (0.04)  &  \bf 0.298 (0.04)  &   0.390 (0.03)  &   0.303 (0.04)\\
			 &  micromass  &   10.01 (2.05)  &   2.518 (0.52)  &   1.263 (0.11)  &   1.005 (0.14)  &   1.235 (0.27)  &  \bf 0.959 (0.19)\\
			 &  letter  &   4.869 (0.27)  &   0.924 (0.04)  &   0.563 (0.02)  &   \bf 0.446 (0.03)  &   0.557 (0.02)  &   0.449 (0.03)\\
			 &  devanagari  &  3.424 (0.28)  & 1.030 (0.04)  &  2.266 (0.17)  &  \bf 0.710 (0.02)  & 1.977 (0.12) &  0.735 (0.02)\\
			 \cmidrule{2-8}
			 &  RCV1  &  1.964 (0.02)  &   1.029 (0.00)  &   0.932 (0.01)  &   \bf 0.716 (0.01)  &   0.862 (0.00)  &   0.745 (0.01)\\
			 &  sector  &  3.631 (0.20)  &   2.910 (0.11)  &   2.672 (0.03)  &  \bf 2.036 (0.03)  &   2.597 (0.05)  &   2.208 (0.07)\\
			 &  ALOI  &   4.443 (0.26)  &   2.512 (0.05)  &   4.889 (0.03)  &  \bf 1.056 (0.02)  &   4.284 (0.04)  &  1.173 (0.03)\\
			 &  ILSVR2010	 &   6.553 (0.10)  &   5.863 (0.00)  &   5.643 (0.00)  &   5.219 (0.00)  &   5.452 (0.00)  &   \bf 5.201 (0.00)\\
			 &  ODP-5K  &  7.733 (0.04)  &  7.192 (0.00)  &  7.120 (0.00)  &  6.603 (0.00)  &  6.981 (0.00) & \bf 6.576 (0.00) \\
		\bottomrule
	\end{tabular}
\end{table*}
\begin{table*}[!h]
	\centering
	\caption{Classification accuracy of nested dichotomies with poorly calibrated base classifiers.\label{tab:acc_poor}}
	\fontsize{9.5pt}{10.5pt} \selectfont
	\begin{tabular}{llcccccc}
		\toprule
		\bf Base Model & \bf Dataset & \bf Baseline & \bf External VS & \bf Internal PS & \bf Both PS & \bf Internal IR & \bf Both IR \\
		\midrule
		\multirow{9}{*}{Na\"ive Bayes}
			&  optdigits  &   0.719 (0.05)  &   0.749 (0.04)  &   0.719 (0.05)  &   0.735 (0.04)  &   0.774 (0.04)  &  \bf 0.795 (0.04)\\
			 &  micromass  &   0.749 (0.05)  &   0.724 (0.05)  &   0.770 (0.05)  &   0.762 (0.05)  &  \bf 0.772 (0.05)  &   0.756 (0.05)\\
			 &  letter  &   0.329 (0.02)  &   0.364 (0.03)  &   0.318 (0.03)  &   0.365 (0.03)  &   0.376 (0.03)  &  \bf 0.412 (0.03)\\
			 &  devanagari  &  0.202 (0.02)  &   0.224 (0.04)  &   0.167 (0.02)  &   0.269 (0.01)  &   0.265 (0.04)  &   \bf 0.340 (0.01)\\
			 \cmidrule{2-8}
		  	 &  RCV1  &   0.644 (0.04)  &   \bf 0.781 (0.00)  &   0.691 (0.03)  &   0.756 (0.00)  &   0.734 (0.01)  &   0.765 (0.01)\\
		  	 &  sector	 &   0.337 (0.07)  &   \bf 0.772 (0.01)  &   0.633 (0.04)  &   0.737 (0.03)  &   0.692 (0.04)  &   0.690 (0.01)\\
		  	 &  ALOI  &   0.024 (0.00)  &   0.029 (0.00)  &   0.019 (0.00)  &  0.124 (0.00)  &   0.094 (0.00)  &   \bf 0.166 (0.01)\\
		  	 &  ILSVR2010  &   0.009 (0.00)  &   0.015 (0.00)  &   0.014 (0.00)  &   0.019 (0.00)  &  0.021 (0.00)  &   \bf 0.026 (0.00)\\
		  	 &  ODP-5K  &  0.043 (0.01)  &  \bf 0.210 (0.00)  &  0.091 (0.00)  &  0.161 (0.00)  &  0.135 (0.01)  &  0.180 (0.00)  \\
		\midrule 
   		\multirow{9}{*}{Boosted Trees}
			&  optdigits  &   0.888 (0.02)  &   0.883 (0.02)  &  \bf 0.922 (0.01)  &   0.917 (0.01)  &   0.921 (0.01)  &   0.915 (0.01)\\
			 &  micromass  &   0.710 (0.06)  &   0.650 (0.06)  &  \bf 0.741 (0.06)  &   0.729 (0.06)  &   0.737 (0.06)  &   0.727 (0.05)\\
			 &  letter  &   0.859 (0.01)  &   0.851 (0.01)  &   0.888 (0.01)  &   0.883 (0.01)  &   \bf 0.890 (0.01)  &   0.884 (0.01)\\
			 &  devanagari  &  0.103 (0.06)  &  0.710 (0.01) &  0.488 (0.11)  &  \bf 0.793 (0.01)  &  0.636 (0.04) &  0.783 (0.01)\\
			 \cmidrule{2-8}
		  	 &  RCV1  &   0.681 (0.06)  &   0.748 (0.01)  &   0.810 (0.00)  &   \bf 0.814 (0.01)  &   0.810 (0.00)  &   0.807 (0.00)\\
		  	 &  sector  &   0.174 (0.08)  &   0.409 (0.02)  &   \bf 0.603 (0.01)  &   0.576 (0.01)  &   0.578 (0.01)  &   0.552 (0.01)\\
			 &  ALOI  &   0.071 (0.03)  &   0.451 (0.01)  &   0.161 (0.01)  &  \bf 0.743 (0.01)  &   0.368 (0.01)  &   0.723 (0.00)\\
			 &  ILSVR2010  &   0.019 (0.00)  &   0.040 (0.00)  &   0.093 (0.00)  &   \bf 0.102 (0.00)  &   0.097 (0.00)  &  0.100 (0.00)\\
		  	 &  ODP-5K  &  0.032 (0.00)  &  0.038 (0.00)  &  0.055 (0.00)  &  0.068 (0.00)  &  0.062 (0.00)  & \bf 0.072 (0.00) \\
		\bottomrule
	\end{tabular}
\end{table*}

Tables~\ref{tab:logistic_regression_nll} and~\ref{tab:logistic_regression_acc} show the NLL and accuracy respectively of nested dichotomies with logistic regression, before and after external calibration is applied. Logistic regression models are known to be well-calibrated~\cite{niculescu2005predicting}. Vector scaling~\cite{guo2017calibration} is used as the external calibration model. We use a 10\% stratified sample of the training data to train the external calibration model, and the remaining 90\% to build the nested dichotomy including the base classifiers. 

\subsubsection{Discussion.}
Tables~\ref{tab:logistic_regression_nll} and~\ref{tab:logistic_regression_acc} show that, for all datasets, a reduction in NLL is observed after applying external calibration with vector scaling (External VS). For some of the datasets with fewer classes (optdigits and letter), the reduction is modest, but the larger datasets see substantial improvements. Interestingly, for some datasets a large improvement in classification accuracy is also observed, especially for the datasets with more classes. Also, the classification accuracy degrades for ILSVR2010 and micromass, despite a large improvement in NLL.

\subsection{Poorly Calibrated Base Learners}
Tables~\ref{tab:nll_poor} and~\ref{tab:acc_poor} show the NLL and classification accuracy respectively of nested dichotomies trained with poorly calibrated base learners, when different calibration strategies are applied. Specifically, we considered nested dichotomies with na\"ive Bayes and boosted decision trees as the base learners. The calibration schemes compared are internal Platt scaling (Internal PS), internal isotonic regression (Internal IR) and external vector scaling (External VS), as well as each internal calibration scheme in conjunction with external vector scaling (Both PS and Both IR, respectively). Three-fold cross validation is used to produce the training data for the internal calibration models, rather than splitting the training data. This is to ensure that each internal calibration model has a reasonable amount of data points to train on, given that internal nodes near the leaves often have few training data available.  When external calibration is performed, 10\% of the data is held out to train the external calibration model. Note that this means 10\% less data is available to train the nested dichotomy and (if applicable) perform internal calibration.  Gaussian na\"ive Bayes is applied for optdigits, micromass, letter and devanagari, and multinomial na\"ive Bayes is used for RCV1, sector, ALOI, ILSVR2010 and ODP-5K as they have sparse features. We use 50 decision trees with AdaBoost~\cite{freund1996game}, limiting the depth of the trees to three.

\subsubsection{Discussion.}
Tables~\ref{tab:nll_poor} and~\ref{tab:acc_poor} show that applying internal calibration is very beneficial in terms of both NLL and classification accuracy. There is no combination of base learner and dataset for which the baseline gives the best results, and there are very few cases where the baseline does not perform the worst out of every scheme. When na\"ive Bayes is used as the base learner, applying internal calibration with isotonic regression always gives better results than the baseline, and when an ensemble of boosted trees is used as the base learner, applying internal Platt scaling always outperforms the uncalibrated case. It is well known that these calibration methods are well-suited to the respective base learners~\cite{niculescu2005predicting}, and this appears to also apply when they are used in a nested dichotomy. 

External calibration also has a positive effect on both NLL and classification accuracy in most cases compared to the baseline. However, the best results are usually obtained when both internal and external calibration are applied together. For na\"ive Bayes, the smaller datasets as well as the two object recognition datasets (ALOI and ILSVR2010) generally see the best performance for both NLL and classification accuracy when applying internal isotonic regression in conjunction with external calibration. Interestingly, the best results for the three text categorisation datasets (RCV1, sector and ODP-5K) were obtained through external calibration only.

Performing both internal Platt scaling and external calibration also gives the best NLL performance for nested dichotomies with boosted trees in most cases, although the improvement compared to isotonic regression is usually small. However, with boosted trees, performance in terms of classification accuracy is less consistent, often being greater when only internal calibration is applied.

\section{Conclusion}
In this paper, we show that the predictive performance of nested dichotomies can be substantially improved by applying calibration techniques. Calibrating the internal models increases the likelihood that the path to the leaf node corresponding to the true class is assigned high probability, while external calibration can correct for the systematic under-confidence exhibited by nested dichotomies. Both of these techniques have been empirically shown to provide large performance gains in terms of accuracy and NLL for a range of datasets when applied individually. Additionally, when both internal and external calibration are applied together, the performance often improves further. Improvements are especially noticeable when the number of classes is high.

Future work in this domain includes evaluating alternative external calibration methods. In our experiments, we applied vector scaling as it is an efficient and scalable solution for large multiclass tasks. However, when computational resources are available, it is possible that employing a more complex method such as matrix scaling, or isotonic regression with one-vs-rest, could provide superior results. It would also be interesting to investigate whether such calibration measures are as effective for other methods of constructing nested dichotomies than random subset selection~\cite{dong2005ensembles,leathart2016building,wever2018ensembles,melnikov2018effectiveness,leathart2018ensembles}. We expect that the calibration techniques discussed in this paper will transfer to such methods. Lastly, it would be of value to evaluate the effect of selective calibration of layers in nested dichotomies. It may be the case that reductions in training time can be achieved by only calibrating the worst models in the tree with little impact on predictive performance. It may also be beneficial to apply Platt scaling to internal models with few training data points instead of isotonic regression, as isotonic regression is known to overfit on smaller samples.

\section{Acknowledgments} 

This research was supported by the Marsden Fund Council from Government funding, administered by the Royal Society of New Zealand. The authors also thank Steven Lang for helpful discussions.

\fontsize{9pt}{10pt} \selectfont

\bibliography{aaai-nd-calibration.bib}

\begin{thebibliography}{}

\bibitem[\protect\citeauthoryear{Acharya, Pant, and
  Gyawali}{2015}]{acharya2015deep}
Acharya, S.; Pant, A.~K.; and Gyawali, P.~K.
\newblock 2015.
\newblock Deep learning based large scale handwritten devanagari character
  recognition.
\newblock In {\em Proceedings of the International Conference on Software,
  Knowledge, Information Management and Applications},  1--6.
\newblock IEEE.

\bibitem[\protect\citeauthoryear{Agrawal \bgroup et al\mbox.\egroup
  }{2013}]{agrawal2013multi}
Agrawal, R.; Gupta, A.; Prabhu, Y.; and Varma, M.
\newblock 2013.
\newblock Multi-label learning with millions of labels: Recommending advertiser
  bid phrases for web pages.
\newblock In {\em Proceedings of the International World Wide Web Conference},
  13--24.

\bibitem[\protect\citeauthoryear{Bengio, Weston, and
  Grangier}{2010}]{bengio2010label}
Bengio, S.; Weston, J.; and Grangier, D.
\newblock 2010.
\newblock Label embedding trees for large multi-class tasks.
\newblock In {\em Proceedings of the Conference on Neural Information
  Processing Systems},  163--171.

\bibitem[\protect\citeauthoryear{Bennett and Nguyen}{2009}]{bennett2009refined}
Bennett, P.~N., and Nguyen, N.
\newblock 2009.
\newblock Refined experts: improving classification in large taxonomies.
\newblock In {\em Proceedings of the International ACM SIGIR Conference on
  Research and Development in Information Retrieval},  11--18.
\newblock ACM.

\bibitem[\protect\citeauthoryear{Beygelzimer, Langford, and
  Ravikumar}{2009}]{beygelzimer2009error}
Beygelzimer, A.; Langford, J.; and Ravikumar, P.
\newblock 2009.
\newblock Error-correcting tournaments.
\newblock In {\em Proceedings of the International Conference on Algorithmic
  Learning Theory},  247--262.
\newblock Springer.

\bibitem[\protect\citeauthoryear{Bingham and Mannila}{2001}]{bingham2001random}
Bingham, E., and Mannila, H.
\newblock 2001.
\newblock Random projection in dimensionality reduction: applications to image
  and text data.
\newblock In {\em Proceedings of the ACM SIGKDD International Conference on
  Knowledge Discovery and Data Mining},  245--250.
\newblock ACM.

\bibitem[\protect\citeauthoryear{Chang and Lin}{2011}]{chang2011libsvm}
Chang, C.-C., and Lin, C.-J.
\newblock 2011.
\newblock {LIBSVM}: a library for support vector machines.
\newblock {\em ACM Transactions on Intelligent Systems and Technology} 2(3):27.

\bibitem[\protect\citeauthoryear{Choromanska and
  Langford}{2015}]{choromanska2015logarithmic}
Choromanska, A.~E., and Langford, J.
\newblock 2015.
\newblock Logarithmic time online multiclass prediction.
\newblock In {\em Proceedings of the Conference on Neural Information
  Processing Systems},  55--63.

\bibitem[\protect\citeauthoryear{Daum{\'e} \bgroup et al\mbox.\egroup
  }{2017}]{daume17logarithmic}
Daum{\'e}, III, H.; Karampatziakis, N.; Langford, J.; and Mineiro, P.
\newblock 2017.
\newblock Logarithmic time one-against-some.
\newblock In {\em Proceedings of the International Conference on Machine
  Learning},  923--932.
\newblock PMLR.

\bibitem[\protect\citeauthoryear{DeGroot and
  Fienberg}{1983}]{degroot1983comparison}
DeGroot, M.~H., and Fienberg, S.~E.
\newblock 1983.
\newblock The comparison and evaluation of forecasters.
\newblock {\em The Statistician}  12--22.

\bibitem[\protect\citeauthoryear{Dekel and Shamir}{2010}]{dekel2010multiclass}
Dekel, O., and Shamir, O.
\newblock 2010.
\newblock Multiclass-multilabel classification with more classes than examples.
\newblock In {\em Proceedings of the International Conference on Artificial
  Intelligence and Statistics},  137--144.
\newblock PMLR.

\bibitem[\protect\citeauthoryear{Dembczy{\'n}ski \bgroup et al\mbox.\egroup
  }{2016}]{dembczynski2016consistency}
Dembczy{\'n}ski, K.; Kot{\l}owski, W.; Waegeman, W.; Busa-Fekete, R.; and
  H{\"u}llermeier, E.
\newblock 2016.
\newblock Consistency of probabilistic classifier trees.
\newblock In {\em Proceedings of the Joint European Conference on Machine
  Learning and Principles and Practice of Knowledge Discovery in Databases},
  511--526.
\newblock Springer.

\bibitem[\protect\citeauthoryear{Deng \bgroup et al\mbox.\egroup
  }{2009}]{deng2009imagenet}
Deng, J.; Dong, W.; Socher, R.; Li, L.-J.; Li, K.; and Fei-Fei, L.
\newblock 2009.
\newblock Imagenet: A large-scale hierarchical image database.
\newblock In {\em Proceedings of the Conference on Computer Vision and Pattern
  Recognition},  248--255.
\newblock IEEE.

\bibitem[\protect\citeauthoryear{Dietterich and
  Bakiri}{1995}]{dietterich1995solving}
Dietterich, T.~G., and Bakiri, G.
\newblock 1995.
\newblock Solving multiclass learning problems via error-correcting output
  codes.
\newblock {\em Journal of Artificial Intelligence Research} 2:263--286.

\bibitem[\protect\citeauthoryear{Dong, Frank, and
  Kramer}{2005}]{dong2005ensembles}
Dong, L.; Frank, E.; and Kramer, S.
\newblock 2005.
\newblock Ensembles of balanced nested dichotomies for multi-class problems.
\newblock In {\em Proceedings of the European Conference on Principles and
  Practice of Knowledge Discovery in Databases}. Springer.
\newblock  84--95.

\bibitem[\protect\citeauthoryear{Fox}{1997}]{fox1997applied}
Fox, J.
\newblock 1997.
\newblock {\em Applied Regression Analysis, Linear Models, and Related
  Methods}.
\newblock Sage.

\bibitem[\protect\citeauthoryear{Frank and Kramer}{2004}]{frank2004ensembles}
Frank, E., and Kramer, S.
\newblock 2004.
\newblock Ensembles of nested dichotomies for multi-class problems.
\newblock In {\em Proceedings of the International Conference on Machine
  Learning},  39--46.
\newblock ACM.

\bibitem[\protect\citeauthoryear{Freund and Schapire}{1996}]{freund1996game}
Freund, Y., and Schapire, R.~E.
\newblock 1996.
\newblock Game theory, on-line prediction and boosting.
\newblock In {\em Proceedings of the Conference on Computational Learning
  Theory},  325--332.

\bibitem[\protect\citeauthoryear{Friedman}{1996}]{friedman1996another}
Friedman, J.~H.
\newblock 1996.
\newblock Another approach to polychotomous classification.
\newblock {\em Technical Report, Statistics Department, Stanford University}.

\bibitem[\protect\citeauthoryear{Guo \bgroup et al\mbox.\egroup
  }{2017}]{guo2017calibration}
Guo, C.; Pleiss, G.; Sun, Y.; and Weinberger, K.~Q.
\newblock 2017.
\newblock On calibration of modern neural networks.
\newblock In {\em Proceedings of the International Conference on Machine
  Learning},  1321--1330.
\newblock PMLR.

\bibitem[\protect\citeauthoryear{Jiang \bgroup et al\mbox.\egroup
  }{2011}]{jiang2011smooth}
Jiang, X.; Osl, M.; Kim, J.; and Ohno-Machado, L.
\newblock 2011.
\newblock Smooth isotonic regression: A new method to calibrate predictive
  models.
\newblock {\em AMIA Summits on Translational Science Proceedings} 2011:16.

\bibitem[\protect\citeauthoryear{Kumar \bgroup et al\mbox.\egroup
  }{2013}]{kumar2013beam}
Kumar, A.; Vembu, S.; Menon, A.~K.; and Elkan, C.
\newblock 2013.
\newblock Beam search algorithms for multilabel learning.
\newblock {\em Machine Learning} 92(1):65--89.

\bibitem[\protect\citeauthoryear{Leathart \bgroup et al\mbox.\egroup
  }{2017}]{leathart2017probability}
Leathart, T.; Frank, E.; Pfahringer, B.; and Holmes, G.
\newblock 2017.
\newblock Probability calibration trees.
\newblock In {\em Proceedings of the Asian Conference on Machine Learning},
  145--160.
\newblock PMLR.

\bibitem[\protect\citeauthoryear{Leathart \bgroup et al\mbox.\egroup
  }{2018}]{leathart2018ensembles}
Leathart, T.; Frank, E.; Pfahringer, B.; and Holmes, G.
\newblock 2018.
\newblock Ensembles of nested dichotomies with multiple subset selection.
\newblock {\em arXiv preprint arxiv:1809.02740}.

\bibitem[\protect\citeauthoryear{Leathart, Pfahringer, and
  Frank}{2016}]{leathart2016building}
Leathart, T.; Pfahringer, B.; and Frank, E.
\newblock 2016.
\newblock Building ensembles of adaptive nested dichotomies with random-pair
  selection.
\newblock In {\em Proceedings of the Joint European Conference on Machine
  Learning and Principles and Practice of Knowledge Discovery in Databases},
  179--194.
\newblock Springer.

\bibitem[\protect\citeauthoryear{Lichman}{2013}]{lichman2013uci}
Lichman, M.
\newblock 2013.
\newblock {UCI} machine learning repository.

\bibitem[\protect\citeauthoryear{Mah{\'e} \bgroup et al\mbox.\egroup
  }{2014}]{mahe2014automatic}
Mah{\'e}, P.; Arsac, M.; Chatellier, S.; Monnin, V.; Perrot, N.; Mailler, S.;
  Girard, V.; Ramjeet, M.; Surre, J.; Lacroix, B.; et~al.
\newblock 2014.
\newblock Automatic identification of mixed bacterial species fingerprints in a
  maldi-tof mass-spectrum.
\newblock {\em Bioinformatics} 30(9):1280--1286.

\bibitem[\protect\citeauthoryear{Melnikov and
  H{\"u}llermeier}{2018}]{melnikov2018effectiveness}
Melnikov, V., and H{\"u}llermeier, E.
\newblock 2018.
\newblock On the effectiveness of heuristics for learning nested dichotomies:
  an empirical analysis.
\newblock {\em Machine Learning} 107(8-10):1--24.

\bibitem[\protect\citeauthoryear{Mena \bgroup et al\mbox.\egroup
  }{2015}]{mena2015using}
Mena, D.; Monta{\~n}{\'e}s, E.; Quevedo, J.~R.; and Del~Coz, J.~J.
\newblock 2015.
\newblock Using {A}* for inference in probabilistic classifier chains.
\newblock In {\em Proceedings of the International Joint Conference on
  Artificial Intelligence}.

\bibitem[\protect\citeauthoryear{Naeini, Cooper, and
  Hauskrecht}{2015}]{naeini2015obtaining}
Naeini, M.; Cooper, G.; and Hauskrecht, M.
\newblock 2015.
\newblock Obtaining well calibrated probabilities using bayesian binning.
\newblock In {\em Proceedings of the AAAI Conference on Artificial
  Intelligence},  2901--2907.

\bibitem[\protect\citeauthoryear{Niculescu-Mizil and
  Caruana}{2005}]{niculescu2005predicting}
Niculescu-Mizil, A., and Caruana, R.
\newblock 2005.
\newblock Predicting good probabilities with supervised learning.
\newblock In {\em Proceedings of the International Conference on Machine
  Learning},  625--632.
\newblock ACM.

\bibitem[\protect\citeauthoryear{Pedregosa \bgroup et al\mbox.\egroup
  }{2011}]{pedregosa2011scikit}
Pedregosa, F.; Varoquaux, G.; Gramfort, A.; Michel, V.; Thirion, B.; Grisel,
  O.; Blondel, M.; Prettenhofer, P.; Weiss, R.; Dubourg, V.; et~al.
\newblock 2011.
\newblock Scikit-learn: Machine learning in python.
\newblock {\em Journal of Machine Learning Research} 12(Oct):2825--2830.

\bibitem[\protect\citeauthoryear{Platt}{1999}]{platt1999probabilistic}
Platt, J.
\newblock 1999.
\newblock Probabilistic outputs for support vector machines and comparisons to
  regularized likelihood methods.
\newblock {\em Advances in Large Margin Classifiers} 10(3):61--74.

\bibitem[\protect\citeauthoryear{Rifkin and Klautau}{2004}]{rifkin2004defense}
Rifkin, R., and Klautau, A.
\newblock 2004.
\newblock In defense of one-vs-all classification.
\newblock {\em Journal of Machine Learning Research} 5:101--141.

\bibitem[\protect\citeauthoryear{Rodr{\'\i}guez, Garc{\'\i}a-Osorio, and
  Maudes}{2010}]{rodriguez2010forests}
Rodr{\'\i}guez, J.~J.; Garc{\'\i}a-Osorio, C.; and Maudes, J.
\newblock 2010.
\newblock Forests of nested dichotomies.
\newblock {\em Pattern Recognition Letters} 31(2):125--132.

\bibitem[\protect\citeauthoryear{Russakovsky \bgroup et al\mbox.\egroup
  }{2015}]{ILSVRC15}
Russakovsky, O.; Deng, J.; Su, H.; Krause, J.; Satheesh, S.; Ma, S.; Huang, Z.;
  Karpathy, A.; Khosla, A.; Bernstein, M.; Berg, A.~C.; and Fei-Fei, L.
\newblock 2015.
\newblock {ImageNet Large Scale Visual Recognition Challenge}.
\newblock {\em International Journal of Computer Vision (IJCV)}
  115(3):211--252.

\bibitem[\protect\citeauthoryear{Wever, Mohr, and
  H{\"u}llermeier}{2018}]{wever2018ensembles}
Wever, M.; Mohr, F.; and H{\"u}llermeier, E.
\newblock 2018.
\newblock Ensembles of evolved nested dichotomies for classification.
\newblock In {\em Proceedings of the Genetic and Evolutionary Computation
  Conference},  561--568.
\newblock ACM.

\bibitem[\protect\citeauthoryear{Zadrozny and
  Elkan}{2001}]{zadrozny2001obtaining}
Zadrozny, B., and Elkan, C.
\newblock 2001.
\newblock Obtaining calibrated probability estimates from decision trees and
  naive {B}ayesian classifiers.
\newblock In {\em Proceedings of the International Conference on Machine
  Learning},  609--616.
\newblock ACM.

\bibitem[\protect\citeauthoryear{Zadrozny and
  Elkan}{2002}]{zadrozny2002transforming}
Zadrozny, B., and Elkan, C.
\newblock 2002.
\newblock Transforming classifier scores into accurate multiclass probability
  estimates.
\newblock In {\em Proceedings of the ACM SIGKDD International Conference on
  Knowledge Discovery and Data Mining},  694--699.
\newblock ACM.

\bibitem[\protect\citeauthoryear{Zhong and Kwok}{2013}]{zhong2013accurate}
Zhong, W., and Kwok, J.~T.
\newblock 2013.
\newblock Accurate probability calibration for multiple classifiers.
\newblock In {\em Proceedings of the International Joint Conference on
  Artificial Intelligence},  1939--1945.

\end{thebibliography}
\bibliographystyle{aaai}

\end{document}